\documentclass{article}
\usepackage[utf8]{inputenc}
\usepackage{multirow}
\usepackage{amsmath}
\usepackage{graphicx}
\usepackage[bookmarks=false,
 breaklinks=false,pdfborder={0 0 1},backref=section,colorlinks=false]
 {hyperref}
 
\newtheorem{thm}{Theorem}
\newtheorem{lemma}[thm]{Lemma}
\newenvironment{proof}{\paragraph{\ Proof:}}{\hfill$\square$}

 \DeclareGraphicsExtensions{.pdf,.jpeg,.png}

\makeatletter

\providecommand{\tabularnewline}{\\}

\usepackage{times}

\usepackage{amsmath,amsfonts,bm}

\def\eqref#1{equation~\ref{#1}}

\def\1{\bm{1}}

\DeclareMathAlphabet{\mathsfit}{\encodingdefault}{\sfdefault}{m}{sl}
\SetMathAlphabet{\mathsfit}{bold}{\encodingdefault}{\sfdefault}{bx}{n}

\usepackage{url}

\title{DEEP CLUSTERING USING ADVERSARIAL NET BASED
CLUSTERING LOSS}

\author{Kart-Leong Lim \thanks{ Use footnote for providing further information
about author (webpage, alternative address)---\emph{not} for acknowledging
funding agencies.  Funding acknowledgements go at the end of the paper.} \\
Institute of Microelectronics\\
2 Fusionopolis Way\\
Singapore, 138634\\
\texttt{\{limkl\}@ime.a-star.edu.sg}  
}

\makeatother

\begin{document}
\maketitle 
\begin{abstract}
Deep clustering is a recent deep learning technique which combines
deep learning with traditional unsupervised clustering. At the heart
of deep clustering is a loss function which penalizes samples for
being an outlier from their ground truth cluster centers in the latent
space. The probabilistic variant of deep clustering reformulates the
loss using $KL$ divergence. Often, the main constraint of deep clustering
is the necessity of a closed form loss function to make backpropagation
tractable. Inspired by deep clustering and adversarial net, we reformulate
deep clustering as an adversarial net over traditional closed form
$KL$ divergence. Training deep clustering becomes a task of minimizing
the encoder and maximizing the discriminator. At optimality, this
method theoretically approaches the $JS$ divergence between the distribution
assumption of the encoder and the discriminator. We demonstrated the
performance of our proposed method on several well cited datasets
such as MNIST, REUTERS and CIFAR10, achieving on-par or better performance
with some of the state-of-the-art deep clustering methods.
\end{abstract}

\section{Introduction}

Deep neural network such as the autoencoder is applied to many signal
processing domains such as speech and audio processing \cite{8769885,8998226},
image clustering \cite{huang2019deep,lim2020} and medical data processing
\cite{8903437,9181498}. While the latent space of autoencoder is
well suited for dimension reduction through reconstruction loss \cite{song2013auto},
the latter is not optimized for clustering/ classification since class
labels cannot be used in the reconstruction loss \cite{min2018survey}.
A recent autoencoder technique known as the Euclidean distance based
clustering (ABC) \cite{song2013auto}, utilize class label information
in neural network by introducing a loss function known as the deep
clustering loss. The goal is to minimizes the Euclidean distance in
the latent space between samples and partitioning learnt by clustering
algorithms e.g. Kmeans or Gaussian mixture model (GMM). In other words,
the latent space of an autoencoder learnt using reconstruction loss
will look different when we further perform K-means on the latent
space. A recent probabilistic approach to ABC uses KL divergence (KLD)
\cite{lim2020,jiang2017variational} and assumes that both the variational
autoencoder (VAE) \cite{kingma2014stochastic} latent space and clustering
approach are Gaussian distributed. We can compare the similarities
and differences between deep clustering and VAE to better understand
the former. The neural network of VAE and deep clustering are optimized
by backpropagating samples in the latent space to update the encoder
weights. Also, both deep clustering and VAE share a similar goal of
modeling each input image as a sample draw of the Gaussian distribution
in the latent space. However, their Gaussian distributions are different.
VAE enforce samples in the latent space to be closely distributed
by a Gaussian distribution with continuous mean $\mu$ and variance
$\sigma$ i.e. $z\sim\mu+\sigma\cdot\mathcal{N}(0,1)$. Whereas the
latent space of deep clustering approaches a Gaussian mixture model
(GMM) and each Gaussian has a set of fixed parameters i.e. $z\mathcal{\sim N}(\eta_{k},\tau_{k})$,
$\forall k\in K$. The goal of VAE is to enforce all class samples
in the latent space to be distributed by continuous mean and variance.
The goal of deep clustering is to enforce class samples in the latent
space to be distributed by a discrete set of mean and variance i.e.
$N(\eta^{*},\tau^{*})$. Thus, only deep clustering is associated
with clustering, making it more suitable for unsupervised classification.
Deep clustering is not without problem. Mainly, when backpropagating
from a neural network, we require a closed form loss function for
tractability \cite{goodfellow2014generative}. Most deep learning
methods including VAE restrict to the assumption of a KLD between
two Gaussians to obtain a closed form loss function for backpropagating.
When generalizing to $f$-divergence between two Gaussians, closed
form solution does not exist. A well known workaround to this problem
is to approximate the JS divergence (JSD) between two Gaussians using
adversarial net. In fact, adversarial net is quite versatile as shown
in Table 1. For different applications, the discriminator can be seen
as performing a specific task different from the discriminators in
other adversarial net. We distinguish the proposed work from other
``\textbf{adversarial net based X}'' in Table I where \textbf{X}
can be any discriminator task. In GAN, \textbf{X} refers to image
generation. Other instances of ``\textbf{adversarial net based X}''
includes AAE, DASC and etc. In our approach, ``Deep clustering using
adversarial net based clustering loss'' (DCAN), \textbf{X} refers
to clustering. More specifically, in DCAN the JSD between latent space
and clustering is approximated by adversarial net. More details are
discussed in the next section.

\begin{table}
\begin{centering}
\begin{tabular}{|r|c|c|}
\hline 
Adversarial net based X & \multicolumn{2}{c|}{Discriminator}\tabularnewline
\hline 
GAN \cite{goodfellow2014generative} & $\left\{ \begin{array}{c}
if\;x\sim p(data),\;T=1\\
else,\;T=0
\end{array}\right.$ & $x$ is from the dataset?\tabularnewline
\hline 
\multirow{1}{*}{DASC \cite{zhou2018deep}} & $\left\{ \begin{array}{c}
if\;z\sim p(\varsigma^{*}),\;T=1\\
else,\;T=0
\end{array}\right.$ & $z$ is from an assigned subspace?\tabularnewline
\hline 
\multirow{1}{*}{AAE \cite{44904}} & $\left\{ \begin{array}{c}
if\;z\sim\mu+\sigma\cdot\mathcal{N}(0,1),\;T=1\\
else,\;T=0
\end{array}\right.$ & $z$ is from a Gaussian prior?\tabularnewline
\hline 
\multirow{1}{*}{DCAN (proposed)} & $\left\{ \begin{array}{c}
if\;z\sim p(z\mid\theta^{*}),\;T=1\\
else,\;T=0
\end{array}\right.$ & $z$ is from an assigned cluster?\tabularnewline
\hline 
\end{tabular}
\par\end{centering}
\caption{Different strategies of using adversarial net based X}
\end{table}

\subsection{Adversarial net based X}

The optimization of adversarial net \cite{goodfellow2014generative}
in general centers around minimizing the encoder/generator and maximizing
the discriminator. However, as mentioned ``\textbf{adversarial net
based X}'' are versatile and we highlight their similarities and
differences as follows: In GAN, the generator is responsible for generating
an image when given a sample from the latent space. The discriminator
of GAN checks whether the generated image belongs to the distribution
of the original dataset. Similarly, in DCAN we minimize the encoder
and maximize the discriminator but for different purpose. In DCAN,
the encoder receives an image and outputs a sample in the latent space.
GAN trains the generator to be sensitive to the variance of the image
class i.e. small displacement of the sample in the latent space can
result in huge image variance. DCAN trains the encoder to be insensitive
to the class variance i.e. different images from the same class populate
tightly close to the cluster center in the latent space. Both DCAN
and AAE \cite{44904} utilize the adversarial net in totally different
ways despite having similar architecture (encoder acting as generator).
The goal of AAE is identical to VAE i.e. the AAE encoder models a
sample in the latent space as a sum and product of $\mathcal{N}(0,1)$.
On the other hand, the goal of DCAN is not identical to VAE. DCAN
models a sample in the latent space, as the sample draw from a cluster
(e.g. Kmeans or Gaussian mixture model), specifically the sample draw
of a $1-of-K$ cluster. The goal of DASC \cite{zhou2018deep} is subspace
clustering, and it is different from deep clustering. Like principal
component analysis, DASC perform unsupervised learning by decomposing
the dataset into different individual eigenvectors, each eigenvector
as orthogonal to each other as possible. Samples from each classes
should ideally reside on their respective $K^{th}$ subspaces, just
as each sample should ideally fall within their respective $K^{th}$
clusters in deep clustering. In DASC, the latent space of DASC is
trained to behave like subspace clustering. Thus the adversarial net
of all four methods, GAN (models the dataset), AAE (models VAE loss),
DASC (models subspace clustering) and DCAN (models deep clustering)
are optimized in totally different ways as summarized in Table I.

\section{Proposed method}

\begin{figure}
\begin{centering}
\includegraphics[scale=0.35]{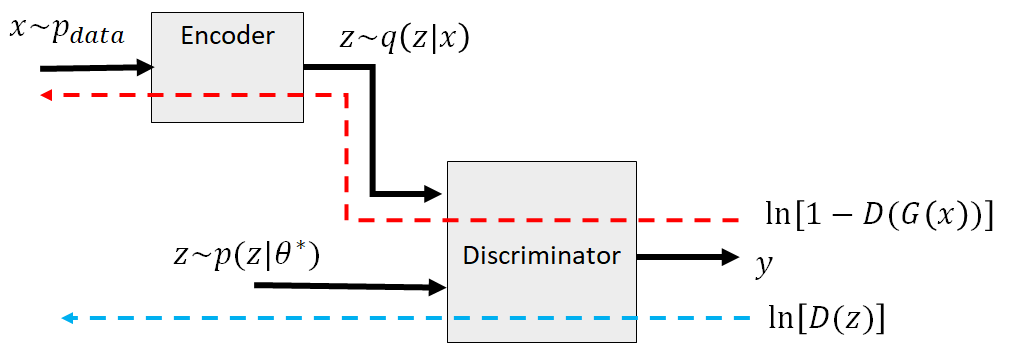}
\par\end{centering}
\caption{Proposed deep clustering loss using adversarial net. It uses the discriminator
and encoder architecture inspired by AAE.}
\end{figure}

One motivation for using a JSD approach is that it does not suffer
from asymmetry unlike KLD \cite{nielsen2019jensen,goodfellow2014generative}.
However, there is no closed form solution available for the JSD between
two data distributions. GAN \cite{goodfellow2014generative} overcome
the need for a closed form solution by employing an adversarial training
procedure. GAN approximates JSD or $D_{JS}\left(p_{data}\parallel p_{g}\right)=E_{x\sim p_{data}}\left[\log D_{G}(x)\right]+E_{x\sim p_{g}}\left[\log\left(1-D_{G}\left\{ x\right\} \right)\right]$
at optimality, where $p_{data}$ is the real data distribution and
$p_{g}$ is the generated data distribution. In the JSD between two
Gaussian distributions for deep clustering, we seek $D_{JS}\left(\:p(z\mid\theta^{*})\parallel q(z\mid x)\:\right)=E_{z\sim p(z\mid\theta^{*})}\left[\ln D(z)\right]+E_{z\sim q(z\mid x)}\left[\ln\left(1-D\left\{ z\right\} \right)\right]$,
whereby $z\sim p(z\mid\theta^{*})$ refers to the sample distributed
from clustering and $z\sim q(z\mid x)$ refers to the sample distributed
from the encoder. Unlike the discriminator of GAN in the original
space, the discriminator of DCAN works in the latent space. We can
visualize $D_{JS}\left(\:p(z\mid\theta^{*})\parallel q(z\mid x)\:\right)$
using Fig. 1. The discriminator tries to distinguish samples from
clustering, $z\sim p(z\mid\theta^{*})$ versus samples generated from
the encoder, $z\sim q(z\mid x)$. 

\subsection{Training DCAN}

We train DCAN by feeding it with $1...N$ positive and negative samples.
We refer to each $n^{th}$ sample of $x$ as $x^{\left(n\right)}$
with dimension $I$. In the forward pass in Fig 1, samples enter the
discriminator through the encoder and clustering. Samples generated
from clustering$^{1}$ are positive i.e. $T=1$ and samples generated
from encoder are negative i.e. $T=0$. We define the neural network
output and target label as $y=\left\{ 0,1\right\} $ and $T=\left\{ 0,1\right\} $
respectively. Both $y$ and $T$ have the same dimension $d$. We
update the discriminator weight using DCAN loss in eqn (1). Since
the samples in the latent space displaces each time the encoder weight
changes, clustering parameters $\theta_{k}=\{\eta_{k},\tau_{k},\pi^{(n)}\forall k\}$
has to be re-computed each time the discriminator weight is updated.
The encoder $q(z\mid x)$ maps $x$ to $z$ in the latent space. However,
we desire to minimize between $q(z\mid x)$ and $p(z\mid\theta^{*}).$
Meaning that we desire to produce a sample $z$ from $q(z\mid x)$
that is as close as possible to the sampling from $p(z\mid\theta^{*})$.
Similarly, DCAN loss in eqn (1) is used to update the encoder weight.
\begin{equation}
\begin{array}{c}
L_{DCAN}=\left\{ \frac{1}{N}\sum_{n=1}^{N}\;\underset{x^{\left(n\right)}\sim p_{data}}{E}\left[\ln\left(1-D\left(G(x^{\left(n\right)})\right)\right)\right]\right\} _{T=0}\\
+\left\{ \frac{1}{N}\sum_{n=1}^{N}\;\underset{z^{\left(n\right)}\sim p(z|\theta^{*})}{E}\left[\ln D\left(z^{\left(n\right)}\right)\right]\right\} _{T=1}
\end{array}
\end{equation}

\section{Experiments}

We selected some benchmarked datasets used by deep clustering for
our experiments in Tables 2. There are several factors affecting deep
clustering. i) The clustering algorithm used in the latent space e.g.
Kmeans or GMM. ii) The number of hidden layers and nodes of the neural
network. iii) The gradient ascent method used for training the loss
functions. iv) The activation functions used. For DCAN, we use $tanh$
for the hidden layers and $sigmoid$ for the output. We set the clustering
iteration to one in Table 2 and a sufficient statistics of at least
600 points in the raw latent space to estimate the cluster parameters.
The number of cluster is the same as the number of classes. We use
a minibatch size of $N=16$ samples, 500 iterations and we use gradient
ascent with momentum with a learning rate between $10^{-3}$ and $10^{-5}$.
We use accuracy (ACC) \cite{CHH05} to evaluate the performance of
our clustering.  \textbf{MNIST:} A well cited 10 digit classes
dataset with no background but it has more samples than USPS. The
train and test sample size is at 50K and 10K. Our settings are $196-384-256$
for the encoder and $256-16-1$ for the discriminator. The input dimension
is originally at $28\times28$ but we have downsampled it to $14\times14$
and flattened to 196. In Table 2, DCAN using raw pixel information
obtained $0.8565$ for ACC which is on par or better than 
deep clustering methods including DEC, ABC and DC-GMM. However
on MNIST, most deep clustering methods could not obtain the performance
of recent deep learning methods such as KINGDRA-LADDER-1M (a.k.a.
KINGDRA) and IMSAT as their goals are much more ACC result oriented.
For KINGDRA, it combines several ideas together including ladder network
\cite{rasmus2015semi} and psuedo-label learning \cite{lee2013pseudo}
for the neural network and using ensemble clustering. In DAC, their
approach is based on pairwise classification. In IMSAT-RPT, the neural
network was trained by augmenting the training dataset. In our case,
we simplified our focused on deep clustering using adversarial net.
DCAN penalizes samples for being an outlier from their cluster centers
and does so uniquely by using the discriminator to penalize the encoder.
The original objective of deep clustering is specifically aimed at
clustering. In some datasets, the latent space of this objective do
not necessarily co-serve as the most important factor for raw pixel
feature extraction. Thus, we also used ResNet18 pretrained model as
a feature extractor for MNIST, prior to deep clustering. This simple
step allows the ACC for DCAN to improved to 0.995 (not shown in Table
2) and in fact outperforms KINGDRA.\textbf{ Reuters-10k: }The Reuters-10k dataset
according to \cite{tian2017deepcluster,xie2016unsupervised} contains
10K samples with 4 classes and the feature dimension is 2000. End-to-end
learning is performed on this dataset. We compared DCAN to other unsupervised
deep learning approaches. DCAN uses $2000-100-500$ and $500-100-1$
respectively for the encoder and discriminator. In Table 2, we observed
that Kmeans alone can achieve an ACC of 0.6018 on the raw feature
dimension. Most deep learning methods can achieve between 0.69 to
0.73 on this dataset. DCAN was able to obtain the best ACC at 0.7867.\textbf{
CIFAR10: }A 10 classes object categorization dataset with 50K training
and 10K testing of real images. On CIFAR10, most  methods have difficulty performing end-to-end learning. As compared to deep
ConvNet which utilizes convolutional neural network and supervised
learning for a specific task (i.e. feature extraction), deep clustering
typically rely on fewer hidden layer in the encoder for end-to-end
learning (i.e. both feature extraction and clustering). Thus, we followed
the experiment setup in KINGDRA \cite{Gupta2020Unsupervised} by using the
``avg pool'' layer of ResNet50 \cite{he2016deep} (ImageNet pretrained)
as a feature extractor with a dimension of 2048. In Table 2, DCAN
uses $2048-1024-512-128$ and $128-16-1$ respectively for the encoder
and discriminator. Despite the stronger performance of DAC, IMSAT
and KINGDRA on MNIST, we outperformed these comparisons on CIFAR10
at $0.5844$. This is despite the fact that both KINGDRA and IMSAT
also use ResNet50 (ImageNet pretrained) for CIFAR10.

\begin{table}
\begin{centering}
\begin{tabular}{|rr|cccc|}
\hline 
Approach &  &  & MNIST &  Reuters10k & CIFAR10\tabularnewline
\hline 
ABC \cite{song2013auto} &  &  & 0.760  & 0.7019 & 0.435\tabularnewline
DEC \cite{xie2016unsupervised} &  &  & 0.843  & 0.7217 & 0.469\tabularnewline
DC-GMM \cite{tian2017deepcluster} &  &  & 0.8555 & 0.6906 & -\tabularnewline
AAE \cite{44904} &  &  & 0.8348  & 0.6982 & -\tabularnewline
IMSAT-RPT \cite{hu2017learning} &  &  & 0.896  & 0.719 & 0.455\tabularnewline
KINGDRA \cite{Gupta2020Unsupervised} &   & & \textbf{0.985}   & 0.705 & 0.546\tabularnewline
DCAN (proposed) &  &  & 0.8565 & \textbf{0.7867} & \textbf{0.5844}\tabularnewline
\hline 
\end{tabular}
\par\end{centering}
\centering{}\caption{ACC Benchmark}
\end{table}

\section{Conclusion}

The objective of deep clustering is to optimize the autoencoder latent
space with clustering information. The key to this technique lies
in a loss function that minimizes both clustering and encoder in the
latent space. Despite the recent success of deep clustering, there
are some concerns: i) How to use probabilistic approach for the loss
function? ii) Can we backpropagate the neural network when there is
no closed form solution for the probabilistic function? A JS divergence
approach does not suffer from asymmetry like KL divergence loss. Unlike
the latter, there is no closed form solution to the JS divergence
when backpropagating the network. This paper addresses the above two
concerns by proposing a novel deep clustering approach based on adversarial
net to estimate JS divergence for deep clustering.

\bibliographystyle{elsarticle-num}
\phantomsection\addcontentsline{toc}{section}{\refname}\bibliography{allmyref}

\section{Appendix}

\subsection{Deep clustering}

The simplest form of deep clustering uses the Euclidean distance loss
in the Autoencoder based clustering (ABC) \cite{song2013auto} approach.
In ABC, the deep clustering objective optimizes the neural network
to be affected by Kmeans clustering or $\mathcal{L}_{2}$ loss function
in eqn (1). On the contrary, traditional reconstruction loss or $\mathcal{L}_{1}$
does not utilize any information from Kmeans.
\begin{equation}
\begin{array}{c}
\underset{w,b}{min}\;\left\{ \mathcal{L}_{1}+\lambda\cdot\mathcal{L}_{2}\right\} \\
=\underset{w,b}{max}\;-\frac{1}{2}\left(T-y\right)^{2}-\lambda\left\{ -\frac{1}{2}\left(h_{z}-\eta^{*}\right)^{2}\right\} 
\end{array}
\end{equation}

\subsection{Probabilistic variant of deep clustering}

Training a deep clustering algorithm such as ABC in eqn (2) requires
a closed-form equation for the clustering loss function. In higher
dimensional latent space, the Euclidean function will become less
robust due to the curse of dimensionality. The probabilistic approach
of ABC is based on the ``KLD between two Gaussians'' in \cite{lim2020,jiang2017variational}.
When the probabilistic approach is no longer a ``KLD between two Gaussians'',
the main problem is that: 
\begin{quote}
a) the loss function can become non-trival to differentiate for backprogation. 
\end{quote}
Instead, we propose a new way to train deep clustering without facing
such constrains i.e. we discard the use of probabilistic approach
in deep clustering. Instead, we approximate deep clustering using
adversarial net. There exist a relationship between deep clustering and adversarial net in two simple steps: 
\begin{quote}
a) \textbf{From ABC to KLD to JSD:} We show that under certain condition
(i.e. unit variance assumption), ABC is identical to KLD and thus
related to JSD. Both KLD and JSD are also under the family of the
$f$-divergence.
\\
b) \textbf{Adversarial net as JSD:} We approximate JSD using adversarial
net. This is possible since adversarial net approaches JSD when the
training becomes optimal. 
\end{quote}
Our main goal is to find a relationship between adversarial net and
deep clustering. First, we establish the relationship between
KLD and ABC as seen in \cite{lim2020}. KLD measures the probabilistic
distance between two distributions. The distributions of the latent
space and deep clustering space are defined as $q(z\mid x)$ and $p(z\mid\theta)$ respectively. Deep clustering space is the latent
space partitioned using GMM or Kmeans. Whereby in GMM, $z=\left\{ z^{(n)}\right\} _{n=1}^{N}\in\mathbb{R}^{Z}$
and $\theta=\{\eta,\tau,\zeta\}$, which are the mean $\eta$, precision
$\tau$ and assignment parameter $\boldsymbol{\zeta}$ for
$K$ number of clusters and $N$ number of samples. Specifically, $\eta=\left\{ \eta_{k}\right\} _{k=1}^{K}\in\mathbb{R}^{Z},\tau=\left\{ \tau_{k}\right\} _{k=1}^{K}\in\mathbb{R}^{Z},\zeta_{k}^{(n)}\in\{0,1\}$
and $\zeta^{(n)}$ is a $1-of-K$ vector. Thus, the KLD based clustering
loss in \cite{lim2020} as follows:
\begin{equation}
\begin{array}{c}
D_{KL}\left(\;p(z\mid\theta)\parallel q(z\mid x)\;\right),\\
p(z\mid\theta)=\prod_{k=1}^{K}\mathcal{N}(z\mid\eta_{k},\left(\tau_{k}\right)^{-1})^{\zeta_{k}},\\
q(z\mid x)=\mu+\sigma\cdot\mathcal{N}(0,1)
\end{array}
\end{equation}

The problem is how do we define KLD in terms of two Gaussians, instead
of a Gaussian and a GMM in eqn (3), which is intractable. To overcome
this, we re-express the GMM term in $p(z\mid\theta)$ as follows
\begin{equation}
\begin{array}{c}
p(z\mid\theta^{*})=\underset{k}{\arg\max\;}p(z\mid\theta_{k})\\
=\mathcal{N}(z\mid\eta^{*},\left(\tau^{*}\right){}^{-1})
\end{array}
\end{equation}

We define $p(z\mid\theta^{*})$ as the $k^{th}$ Gaussian component
of Kmeans or GMM that generates sample $z$ in the latent space. i.e.
we use $\zeta$ ( cluster assignment) to compute $\eta^{*}$ (mean) and $\tau^{*}$ (precision). Thus, the KLD of GMM and latent space in eqn (5) can
now become a KLD between two Gaussian distributions. Also, a closed
formed equation is available for the latter.
\begin{equation}
\begin{array}{c}
D_{KL}\left(\:\mathcal{N}(z\mid\eta^{*},\left(\tau^{*}\right){}^{-1})\parallel\mathcal{N}(z\mid\mu,\varrho^{2})\:\right)\end{array}
\end{equation}

The relationship between KLD and ABC can be explained in Lemma 1: If
we discard the second order terms i.e. $\varrho^{2}$ and $\tau$
in eqn (7), the KLD reverts back to the original ABC loss \cite{song2013auto,yang2017towards}
in eqn (2).
\begin{equation}
D_{KL}\left(\mathcal{N}(z\mid\eta^{*})\parallel\mathcal{N}(z\mid\mu)\right)=\frac{1}{2}\left(\eta^{*}-\mu\right)^{2}
\end{equation}

To relate KLD to JSD, we simply recall JSD as the averaging between
two KLDs in eqn (8). We next discuss how to relate adversarial net
to JSD.

\begin{equation}
\begin{array}{c}
D_{JS}\left(\:p(z\mid\theta^{*})\parallel q(z\mid x)\:\right)\\
=\frac{1}{2}D_{KL}\left(p\parallel\frac{p+q}{2}\right)+\frac{1}{2}D_{KL}\left(q\parallel\frac{p+q}{2}\right)
\end{array}
\end{equation}

\begin{lemma}\label{label of lemma}
     
Relating deep clustering  in eqn (2) to probabilistic deep clustering in eqn (5): We show that under assumption
of ``unit variance'', the $KL$D between the encoder latent space
and GMM reverts back to the ABC loss. 

\end{lemma}

\begin{proof}

For a $D_{KL}$ between two Gaussian distributions, there is a unique
closed-form expression available. When we assume unit variance i.e.
$\tau=\sigma=1$, the $D_{KL}$ reverts back to the original Euclidean
distance based clustering loss in eqn (1) as follows 
\begin{equation}
\begin{array}{c}
D_{KL}\left(p(z\mid\theta)\parallel q(z\mid x)\right)\;\;s.t.\:\:\{\tau=\sigma=1\}\\
=D_{KL}\left(\mathcal{N}(z_{n}\mid\eta^{*},\left(\tau^{*}\right){}^{-1})\parallel\mathcal{N}(z_{n}\mid\mu,\sigma)\right)\\
=\ln\tau^{*}+\ln\sigma+\frac{\left(\tau^{*}\right)^{-1}+\left(\eta^{*}-\mu\right)^{2}}{2\sigma^{2}}-\frac{1}{2}\\
=\frac{1}{2}\left(\eta^{*}-\mu\right)^{2}
\end{array}
\end{equation}

\end{proof}

\subsection{Deep clustering using Adversarial net approaches JSD }

The problem with using JSD for deep clustering is that there is rarely
a closed form solution available for a ``JSD between two Gaussian
distributions''. GAN overcome this by employing an adversarial training
procedure that approximates $D_{JS}\left(p_{data}\parallel p_{g}\right)$
at optimality, where $p_{data}$ is the real data distribution and
$p_{g}$ is the generated data distribution as follows \cite{goodfellow2014generative}:
\begin{equation}
\begin{array}{c}
D_{JS}\left(p_{data}\parallel p_{g}\right)\\
=\underset{x\sim p_{data}}{E}\left[\log D_{G}(x)\right]+\underset{z\sim p(z)}{E}\left[\log\left(1-D_{G}\left\{ G(z)\right\} \right)\right]\\
=\underset{x\sim p_{data}}{E}\left[\log D_{G}(x)\right]+\underset{x\sim p_{g}}{E}\left[\log\left(1-D_{G}\left\{ x\right\} \right)\right]
\end{array}
\end{equation}

Unlike the discriminator of GAN in the original space, the discriminator
of DCAN works in the latent space. Despite that, we can easily reformulate
GAN into DCAN as follows:
\begin{equation}
\begin{array}{c}
D_{JS}\left(p(z\mid\theta^{*})\parallel q(z\mid x)\right)\\
=\underset{z\sim p(z\mid\theta^{*})}{E}\left[\ln D(z)\right]+\underset{x\sim p_{data}}{E}\left[\ln\left(1-D\left\{ G(x)\right\} \right)\right]\\
=\underset{z\sim p(z\mid\theta^{*})}{E}\left[\ln D(z)\right]+\underset{z\sim q(z\mid x)}{E}\left[\ln\left(1-D\left\{ z\right\} \right)\right]
\end{array}
\end{equation}
Whereby $E_{z\sim p(z\mid\theta^{*})}\left[\cdot\right]$ refers to
the expectation function where the sample is distributed from deep
clustering space and $E_{z\sim q(z\mid x)}\left[\cdot\right]$ refers
to the expectation function where the sample is distributed from the
latent space. We refer to Lemma 2 for the claim on eqn (10).

\begin{lemma}\label{label of lemma}

The adversarial net based deep clustering loss by DCAN can be shown to approach JSD
at optimum: 

\end{lemma}

\begin{proof}

For the sake of brevity, we refer to $p(z\mid\theta^{*})$ and $q(z\mid x)$
as $p$ and $q$ respectively. Optimal discriminator occurs when
$G$ is fixed, i.e. $D=\frac{p}{p+q}$.  Substituting $D$, we define
the LHS and RHS below. Lastly, if we subject $p=q=1$ on both sides, we can validate the claim on eqn (10).

\begin{equation}
\begin{array}{c}
\boldsymbol{LHS}:E_{z\sim p}\left[\ln D(z)\right]+E_{x\sim p_{data}}\left[\ln\left(1-D\left\{ G(x)\right\} \right)\right]\;\;s.t.\:\:\{D=\frac{p}{p+q}\}\\
=E_{z\sim p}\left[\log D(z)\right]+E_{z\sim q}\left[\log\left(1-D(z)\right)\right]\;\;s.t.\:\:\{D=\frac{p}{p+q}\}\\
=E_{z\sim p}\left[\ln\frac{p}{p+q}\right]+E_{z\sim q}\left[\ln\frac{q}{p+q}\right]\\
=\int_{z}\:p\ln\frac{p}{p+q}+q\ln\frac{q}{p+q}\:dz\\
\\
\boldsymbol{RHS}:D_{JS}\left(p\parallel q\right)=\frac{1}{2}\int\:p\ln\frac{2\cdot p}{p+q}+q\ln\frac{2\cdot p}{p+q}\:dz\\
=\frac{1}{2}\int_{z}\:p\left\{ \ln\frac{p}{p+q}+\ln2\right\} +q\left\{ \ln\frac{q}{p+q}+\ln2\right\} \:dz
\end{array}
\end{equation}
\\
Thus, $LHS\leq2RHS-2\log2$.

\end{proof}

\end{document}